\documentclass[11pt]{article}

\usepackage{natbib}

\usepackage{amsmath}
\usepackage{amsfonts}
\usepackage{amssymb}
\usepackage{amsthm}

\usepackage{hyperref}       
\usepackage[dvipsnames]{xcolor}
\usepackage{url}            
\usepackage{booktabs}       
\usepackage{amsfonts}       
\usepackage{nicefrac}       
\usepackage{microtype}      
\usepackage{bbm}
\usepackage{pgfplots}
\usepackage{thmtools,thm-restate}
\usepackage{xspace}
\usepackage{authblk}
\usepackage{caption}
\usepackage{subcaption}
\usepackage{tikz}
\usepackage{paralist}
\usepackage{libertine}
\usepackage[nameinlink]{cleveref}
\usepackage{enumitem}

\usepackage{amsmath}
\usepackage{amssymb}
\usepackage{amsthm}
\usepackage{mathtools}
\usepackage{comment}
\usepackage{todonotes}
\usepackage{float}
\usepackage[algo2e,ruled,linesnumbered]{algorithm2e}

\usepackage{thmtools,thm-restate}
\usepackage{physics}
\newcommand{\error}[1]{ \left| \mathbb{E}_{(x,y) \sim \mathcal{D}} \ [\one (#1(x) \neq y)] - \ \mathbb{E}_{(x,y) \sim \mathcal{D}} \ [\one (h^*(x) \neq y)] \right|}

\newcommand{\hhat}{\hat{h}}
\newcommand{\hstar}{h^*}
\newcommand{\hclass}{\mathcal{H}}

\newcommand{\DAC}{\widetilde{\mathcal{D}}_A}
\newcommand{\DBC}{\widetilde{\mathcal{D}}_B}
\newcommand{\RA}{P_{(x,y) \sim \dist } [x \in A]}

\newcommand{\normalF}{F}
\newcommand{\corruptF}{\widetilde{F}}

\usepackage{mkolar_definitions}
\usepackage{multirow}
\usepackage{amsmath}

\newcommand{\dist}{\mathcal{D}}
\newcommand{\DA}{\mathcal{D}_A}
\newcommand{\DB}{\mathcal{D}_B}
\newcommand{\closure}{cl(\mathcal{H})}
\newcommand{\Bern}{\text{Bernoulli}}

\usepackage{color-edits}[suppress]
\addauthor[suppress]{pco}{purple}
\addauthor[suppress]{kms}{orange}
\addauthor[suppress]{as}{cyan}
\usepackage{alg}
\usepackage{multirow}
\usepackage{graphicx}

\usepackage[utf8]{inputenc} 
\usepackage{comment}

\usepackage[margin=1in]{geometry}
\usepackage{txfonts}
\usepackage[T1]{fontenc}

\title{On the Vulnerability of Fairness Constrained Learning to Malicious Noise}
\author[1]{Avrim Blum}
\author[1]{Princewill Okoroafor}
\author[1]{Aadirupa Saha}
\author[1]{Kevin Stangl}
\affil[1]{Toyota Technological Institute, Chicago\protect\\
{\footnotesize \texttt{pco9@cornell.edu}}}
\date{}

\begin{document}

\maketitle

\begin{abstract}
We consider the vulnerability of fairness-constrained learning to small amounts of malicious noise in the training data. \cite{lampert} initiated the study of this question and presented negative results showing there exist data distributions where for several fairness constraints, any proper learner will exhibit high vulnerability when group sizes are imbalanced. Here, we present a more optimistic view, showing that if we allow randomized classifiers, then the landscape is much more nuanced. For example, for Demographic Parity we show we can incur only a $\Theta(\alpha)$ loss in accuracy, where $\alpha$ is the malicious noise rate, matching the best possible even without fairness constraints. For Equal Opportunity, we show we can incur an $O(\sqrt{\alpha})$ loss, and give a matching $\Omega(\sqrt{\alpha})$ lower bound. In contrast, \cite{lampert} showed for proper learners the loss in accuracy for both notions is $\Omega(1)$. The key technical novelty of our work is how randomization can bypass simple "tricks" an adversary can use to amplify his power.  
We also consider additional fairness notions including Equalized Odds and Calibration. For these fairness notions, the excess accuracy clusters into three natural regimes $O(\alpha)$,$O(\sqrt{\alpha})$, and $O(1)$. These results provide a more fine-grained view of the sensitivity of fairness-constrained learning to adversarial noise in training data.
\end{abstract}

\section{Introduction}

The widespread adoption of machine learning algorithms across various domains, including recidivism prediction \cite{flores2016false,dieterich2016compas}, credit lending \cite{Kozodoi_2022}, 
and predictive policing \cite{lum2016predict}, 
has raised significant concerns regarding biases and unfairness in these models. Consequently, substantial efforts have been devoted to developing approaches for learning fair classification models that exhibit effective performance across protected attributes such as race and gender.

One critical aspect of addressing fairness in machine learning is ensuring the robustness of models against small amounts of adversarial corruption present in the training data. 
This data corruption may arise due to flawed data collection or cleaning processes \cite{saunders2013accuracy}, strategic misreporting \cite{hardt2016strategic}, under-representation of certain subgroups \cite{blum2019recovering}, or distribution shift over time \cite{schrouff2022maintaining}.

Empirical studies have demonstrated that such data unreliability is often centered on sensitive groups e.g. \cite{gianfrancesco2018potential}, emphasizing the need to understand the vulnerability of fair learning to adversarial perturbations. 
A concerning possibility is that fairness constraints might allow the adversary to amplify the effect of their corruptions by exploiting how these constraints require the classifier to have comparable performance on every relevant sub-group, even small ones.

Previous work by \cite{lampert} and \cite{celis2021fair} have explored this topic from a theoretical perspective, considering different adversarial noise models. 
\cite{celis2021fair} focused on the $\eta$-Hamming model, where the adversary selectively perturbs a fraction of the dataset by modifying the protected attribute. 

\cite{lampert} on the other hand, investigated the Malicious Noise model, where an $\alpha$ fraction of the data-set (or distribution) 
is uniformly chosen and those data points are arbitrarily perturbed by the adversary.
We will focus on this Malicious Noise model.
In our study, we extend the framework of fair learning in the presence of Malicious Noise \cite{lampert} by  considering a broader range of fairness constraints and introducing a way to bypass some of their negative results by randomizing the hypothesis class.

\cite{lampert} present a pessimistic outlook, highlighting data distributions in which any proper learner, particularly in scenarios with imbalanced group sizes, exhibits high vulnerability to adversarial corruption when the learner is constrained by Demographic Parity \cite{calders2009building} or Equal Opportunity \cite{hardt16}. 
These results demonstrate novel and concerning challenges to designing fair learning algorithms resilient to adversarial manipulation in the form of 
Malicious Noise. 

The results of \cite{lampert} indicate that fairness constrained learning is much less robust than unconstrained learning.

In this paper, we  present a more optimistic perspective on the vulnerability of fairness-constrained learning to malicious noise by introducing randomized classifiers. 
By allowing randomized classifiers, we can explore alternative strategies that effectively mitigate the impact of malicious noise and enhance the robustness of fairness-constrained models.
In addition, we extend the analysis beyond the fairness constraints examined in \cite{lampert}, providing a complete characterization of the robustness of each constraint and revealing a diverse range of vulnerabilities to Malicious Noise.

\subsection{Our Contributions}
We bypass the impossibility results in \cite{lampert} by allowing the learner to produce a randomized improper classifier. This classifier is constructed from hypotheses in the base class $\mathcal{H}$ using our post-processing procedure, which we refer to as the $(P,Q)$-Randomized Expansion of a hypothesis class $\mathcal{H}$, or $\PQ$
\begin{definition}[$\PQ$] \label{defn:pq}
For each classifier $h \in \mathcal{H}$, for $p,q \in[0,1]$ 
\begin{align*}
    h_{p,q}(x) : = 
    \begin{cases}
    h(x)\quad \text{ with probability }1-p \\
    y \sim \Bern(q) \quad \text{otherwise}
    \end{cases}
\end{align*}
We define $\PQ$ as the expanded hypothesis class created by the set of all possible
$h_{p,q}(x)$.
\[ 
\PQ := \{ h_{p,q} \mid h \in \calH, p,q \in [0,1]\} 
\]
\end{definition}
When clear from context we drop the dependence on $p,q$ and simply refer to $\hat{h} \in \PQ$.

Larger $p$ means we ignore more of the information in the base classifier $h$ and rely on the $\Bern(q)$.
The main technical questions we address in this paper are:
\begin{center}
 \emph{
How susceptible and sensitive are fairness constrained learning algorithms
to Malicious Noise and to what extent does this vulnerability
depend on the specific fairness notion, especially if we allow improper learning?}


\end{center}

%

We focus on proving the existence of $h^{'} \in \PQ$ that satisfies a given fairness constraint and exhibits minimal accuracy loss on the original data distribution. 
Recall that $\alpha$ is the fraction of the overall distribution that is corrupted by the adversary.

Our list of contributions is: 
\begin{enumerate}
\item We propose a way to 
bypass 
lower bounds \cite{lampert} in Fair-ERM with Malicious Noise by extending the hypothesis class using  the $\PQ$ notion.
\item For the Demographic Parity \cite{calders2009building} constraint, our approach guarantees no more than $O(\alpha)$ loss in accuracy (which is \emph{optimal}
in the Malicious Noise model \emph{without} fairness constraints \cite{malnoise}). 
In other words, in contrast to the perspective in \cite{lampert} which shows $\Omega(1)$ accuracy loss, we show that Demographic Parity constrained ERM can be made just as robust to Malicious Noise as unconstrained ERM.
\item For the Equal Opportunity \cite{hardt16} constraint, we guarantee no more than $O(\sqrt{\alpha})$ accuracy loss
and show that this is tight, i.e no classifier can do better. 
\item For the fairness constraints Equalized Odds \cite{hardt16}, Minimax Error \cite{minimaxfair}, Predictive Parity, and our novel fairness constraint Parity Calibration, we show strong negative results.
Namely, for each constraint there exist natural distributions such that an adversary that can force any algorithm to return a fair classifier that has $\Omega(1)$ loss in accuracy.
\item For Calibration \cite{faircalib}, we observe that the excess accuracy loss is at most $O(\alpha)$.
\end{enumerate}

\section{Related Work}
\cite{kearns1988learning} introduced the notion of malicious noise which is analyzed in \cite{bshouty2002pac, auer1998line, klivans2009learning, long2011learning, awasthi2014power}.
\cite{balcan2022robustly} considers a related adversary as a way to formalize data poisoning attacks in adversarial robustness
 \cite{goodfellow2014explaining}.

The interaction of fairness constraints with explicitly unreliable data is a critical research direction
since issues of bias and fairness are often closely connected with data reliability concerns \cite{gianfrancesco2018potential}.
Malicious noise is both a way to model an explicit adversary and a way to consider unknown natural issues with the data distribution.

\cite{celis2021fair}, which also studies fairness with data corruptions. 
primarily focuses on the stronger Nasty Noise Model \cite{bshouty2002pac}
combined with an assumption on the minimum size of groups/events. 
They do not consider how randomized post-processing improves robustness.

The most closely related work to ours, \cite{lampert}, explores the limits of fairness-aware PAC learning within the classic malicious noise model of  \cite{valiant}, where the adversary can replace a uniformly random fraction of the data points with arbitrary data, with full knowledge of the learning algorithm, the data distribution, and the remaining examples. 
\cite{lampert} focuses on binary classification with just two popular group fairness constraints: Demographic Parity  \cite{calders2009building} and Equal Opportunity \cite{hardt16}. In addition to those constraints, we also consider Equalized Odds and multiple Calibration notions.
Similarly to \cite{lampert}, a key aspect of our results is how the size of the smaller group makes it more vulnerable to data corruption.

\subsection{Group Fairness Discussion}

Note that group fairness constraints \cite{chouldechova2017fair,klein16, hardt16} are relatively simple to evaluate and provide relatively weak guarantees in contrast to fairness notions in \cite{calders2009building, dwork2021outcome, hebert2018multicalibration}, among others.  
However, despite this weakness,  these group notions are used in practice \cite{metricsinpractice} to check model performance, so continuing to investigate them in parallel to the stronger fairness notions is worthwhile.

\cite{blum2019recovering} takes a somewhat converse perspective to this paper.
Instead of considering worst case instances for how much fairness constraints force excess accuracy loss with an adversary, that paper
asks how fairness constraints can help us recover from the biased data when the bias is more benign, but still complicates Empirical Risk Minimization.
Future work could connect our results with theirs by considering an intermediate adversary model, such as \cite{massart2006risk}.

Interestingly, there may be high level connections between our paper and work in federated learning with differing data quality levels, e.g \cite{chu2023focus}.


\section{Preliminaries}
\label{sec:prelim}
In fairness-constrained learning, the goal is to learn a classifier that achieves good predictive performance while satisfying certain fairness constraints that
connect the performance of the classifier on multiple groups, to ensure effective performance on all groups.

Specifically, we start with a dataset consisting of examples with feature vectors $(x \in \mathcal{X})$, labels $(y \in \mathcal{Y})$, and group attributes $(z \in \mathcal{Z})$. 
We assume that each example is drawn i.i.d from a joint distribution $\mathcal{D}$ of random variables $(X,Y,Z)$. 
There are multiple groups in the dataset, and we aim to ensure that the classifier's predictions do not unfairly favor or disfavor any particular group. 
We will denote $\mathcal{D}_z$ as the conditional distribution of random variables $X$ and $Y$ given $Z = z$. For simplicity, we will assume there are two disjoint groups: $A$ and $B$ in the dataset with B being the smaller and more vulnerable of the two. However, our results apply more broadly to any number of groups.

We aim to use the dataset to learn a classifier $f : \mathcal{X} \rightarrow \mathcal{Y}$ given a hypothesis class $\mathcal{H}$. 
However, in this paper we suppress sample complexity learning issues and focus on characterizing the
accuracy properties of the best hypothesis in the expanded hypothesis class with a corrupted data distribution $\widetilde{D}$.
The goal is to probe the fundamental sensitivity of Fair-ERM to unreliable data in the large sample limit.

To this end, we consider solving the standard risk minimization problem with fairness constraints, known as Fair-ERM.
\begin{align}
\min_{h\in\mathcal{H}} & ~~\mathbb{E}_{(X,Y,Z)\sim\mathcal{D}} \left[\one(h(X) \neq Y)\right] ~~~~\ \\
\text{subject to} & ~~~~ F_z(h) = F_{z'}(h) \qquad \forall z,z'\in \mathcal{Z}. \label{FairnessConstraint}
\end{align}
where $F_z(h)$ is some fairness statistic of $h$ for group $z$ given the true labels $y$, such as \emph{true positive rate} :
$ \text{(TPR)}: F_z(h) = \mathbb{P}(h(X)=+1|Y=+1,Z=z)$.


We make a mild \emph{realizability assumption} that there exists a solution to this risk minimization problem. That is, there is at least one hypothesis in the class that satisfies the fairness constraint. This optimal solution is denoted as $h^*$.

As noted above, since we allow our hypothesis class to be group-aware, we can reason about $h^*_z$ for all $z \in Z$, where $h^*_z$ is the restriction of the optimal classifier $h^*$ to members of group $z$. In other words, $h_{z}^{*}$ is the optimal group-specific classifier for Group $z$.

 We also make another mild assumption that each group $z$ has at least some constant fraction of positive examples, and at least some constant fraction of negative examples (e.g., the fraction of positive examples in each group is between 5\% and 95\%.   If positive or negative examples are too rare within any group, then some of the fairness notions do not make sense, and we find this assumption reasonable.

For the results in our paper, we only need the assumption that each group has non-trivial fraction of positives.
Formally, we assume that that for each fixed group $A$, where $r_A= P_{(x,y) \sim \mathcal{D} } [x \in A]$,
\begin{equation}
     P_{(x,y) \sim \mathcal{D}} [y = 1 \cap x \in A] := r_{A}^{+} \geq \frac{r_A}{c}
    \label{ass:raplus}
\end{equation}
for some  integer c.
Think $c \leq 20$.
This will allow the adversary to modify each group's true positive rate substantially, but not arbitrarily, because there is some non-trivial fraction of positives in each group.

\subsection{Fairness Notions}
Different formal notions of group fairness have previously been proposed in literature. These notions include, but are not limited to, Demographic Parity, Equal Opportunity, Equalized Odds, Minimax Fairness, and Calibration\cite{dwork2012fairness, calders2009building, hardt16, klein16, chouldechova2017fair}.

Selecting the “right” fairness measure is, in general, application-dependent.\footnote{We would also note that these fairness constraints are imperfect measures of fairness that likely do not capture all of the normative properties relevant to a specific task or system.  }
One of our goals in this work is to provide understanding of their implications under adversarial attack, which could aid in the selection process.
For the convenience of the reader, we include a table in Appendix \ref{fairtable} summarizing the fairness notions we consider in this paper. 
Other than Calibration, these all are notions for binary classifiers. 
In Section \ref{subsec:calib} we will introduce a new variant of Calibration and will defer discussion of that notion until then.

\subsection{Adversary Model}
\label{subsec:adversary}
Throughout this paper, we focus on the Malicious Noise Model, introduced by \cite{malnoise}.
This model considers a worst-case scenario where an adversary has complete control over a uniformly chosen $\alpha$ proportion of the training data and can manipulate that fraction in order to move the learning algorithm towards their desired outcomes, i.e. increasing test time error [on un-corrupted data].
\kmsdelete{However, that $\alpha$ fraction is chosen uniformly from natural examples, unlike in \cite{bshouty2002pac}.}

In \cite{malnoise}'s model, the samples are drawn sequentially from a fixed distribution. With probability $\alpha$ and full knowledge of the learning algorithm, data distribution and all the samples that have been drawn so far, the adversary can replace sample $(x,y)$ with an arbitrary sample $(\tilde{x}, \tilde{y})$.

At each time-step $t$,
\begin{enumerate}
    \item The adversary chooses a distribution $\widetilde{\mathcal{D}}_t$ that is $\alpha-$close to the original distribution $\mathcal{D}$ in Total Variation distance.
    \item The algorithm draws a sample $(x_t,y_t)$ from $\widetilde{\mathcal{D}}_t$ instead of $\mathcal{D}$
\end{enumerate}
Note that the adversary's choice at time $t$, $\widetilde{\mathcal{D}}_t$ can depend on the samples $\{x_1,y_1, \ldots, x_{t-1}, y_{t-1}\}$ chosen so far.

Reframing the Malicious Noise Model in this manner simplifies analysis and allows us to focus on the fundamental aspect of this model which is how the accuracy guarantees of fairness constrained learning change as a function of $\alpha$.

\subsection{Core Learning Problem}
\label{subsec: learningproblem}
In the fair-ERM problem with Malicious Noise, our goal is to find the optimal classifier $h^*$ subject to a fairness constraint. However, the presence of the Malicious Noise makes this objective challenging. 
Instead of observing samples from the true distribution $\mathcal{D}$, we observe samples from a corrupted distribution $\widetilde{\mathcal{D}}$. 

In the standard ERM setting, \cite{kearns1988learning} show that the optimal classifier that can be learned using this corrupted data is one that is $O(\alpha)$-close to $h^*$ in terms of accuracy [on the original distribution]. 
The fair-ERM problem with a Malicious Noise adversary introduces an additional layer of complexity, as we must also ensure fairness while achieving high accuracy. 

\begin{definition}
\label{def:robust}
We say a learning algorithm for the fair-ERM problem is $\beta$-robust with respect to a fairness constraint $F$ in the malicious adversary model with corruption fraction $\alpha$, if it returns a classifier $h$ such that $\widetilde{F}_z(h) = \widetilde{F}_{z'}(h)$ and 
\begin{align*}
| \mathbb{E}_{\mathcal{D}} \ [\one (h(X, Z) \neq Y)] - \mathbb{E}_{\mathcal{D}} \ [\one  (h^{*} (X, Z) \neq Y)] | \leq \beta(\alpha)
\end{align*}
where $h^*$ is the optimal classifier for the fair-ERM problem on the true distribution $\mathcal{D}$ with respect to a hypothesis 
class $\mathcal{H}$ and $\beta$ is a function of $\alpha$.
\end{definition}
This definition captures the desired properties of a learning algorithm that can perform well under the malicious noise model while achieving both accuracy and fairness, as measured by the fairness constraint $F$.

Thus, this is an agnostic learning problem \cite{HAUSSLER199278} with an adversary and fairness constraints.
As referenced in the introduction, we will allow the learner to return $h^{'} \in \PQ$, where $\PQ$ is a way to post-process each
$h \in \mathcal{H}$ using randomness.
In Sections \ref{sec:mainresults} and  \ref{subsec:calib} will characterize the optimal value of $\beta$ given the relevant fairness constraint $F$ and base hypothesis class
$\mathcal{H}$.

\section{Main Results: Demographic Parity, Equal Opportunity and Equalized Odds}
\label{sec:mainresults}
We now present our technical findings for Demographic Parity, Equal Opportunity, and Equalized Odds, and show how randomization enables better accuracy for Fair-ERM with Malicious Noise.
\cite{lampert} show impossibility results for Demographic Parity and Equal Opportunity  where a \emph{proper} learner is forced to return a classifier with $\Omega(1)$ excess unfairness and accuracy
 compared to $h^{*}$ for a synthetic and finite hypothesis class/distribution. 
 
To overcome this limitation, we propose a novel approach to make the hypothesis class $\mathcal{H}$ more robust, by injecting noise into each hypothesis $h \in \mathcal{H}$. In other words, we allow improper learning, and refer to the resulting expanded set of hypotheses as $\PQ$.
By injecting controlled noise into the hypotheses, we effectively ``smooth out" the hypothesis class $\mathcal{H}$, making it more resilient against adversarial manipulation. 

Since we allow group-aware classifiers, we learn two classifiers $h_{A}, h_{B} \in \PQ$, typically distinct from each other.
Our method minimizes fairness loss for any hypothesis class and true distribution $\mathcal{D}$, under the assumption that at least one classifier in the original hypothesis class $\mathcal{H}$ satisfies the fairness constraints. 
We aim to find a fair classifier $\hat{h} \in \PQ$ that is as good as the best $h^{*} \in \mathcal{H}$.

\subsection{Demographic Parity}
Demographic Parity \cite{calders2009building} requires that the decisions of the classifier are independent of the group membership;
 that is, $P_{(x,y) \sim \DA} [h(x)=1] = P_{(x,y) \sim \DB} [h(x)=1]$\footnote{Note there is no reference in the definition to the true labels, so a trivial hypothesis that flips a random coin for all examples would satisfy this notion, albeit at minimal accuracy. }. 

When the original distribution $\mathcal{D}$ is corrupted, a fair hypothesis on $\mathcal{D}$ may seem unfair to the learner. In order to analyze our approach it is important to understand how the fairness violation of a fixed hypothesis changes after the adversary corrupts an $\alpha$ proportion of the distribution.

\begin{restatable}[Parity after corruption]{proposition}{corruptparity}\label{prop:corruptparity}
Let $\widetilde{\mathcal{D}}$ be any corrupted distribution chosen by the adversary, and $h$ be a fixed hypothesis in $\mathcal{H}$. For a fixed group $A$, the following inequality bounds the change in the proportion of positive labels assigned by $h$:
$
\left| P_{(x,y) \sim \DAC} [h(x)=1] - P_{(x,y) \sim \DA} [h(x)=1] \right|  \leq \\
 \frac{\alpha}{(1-\alpha) r_A + \alpha }
$
where $r_A = \RA$, i.e how prevalent the group is in the original distribution. 
\end{restatable}


This proposition provides an upper bound on the change in the proportion of positive labels assigned by a fixed hypothesis $h$ in $\mathcal{H}$ after the distribution has been corrupted according to the Malicious Noise Model. The full proof can be found in the Appendix \ref{proof:corruptparity}. 
The proof shows that this change is bounded by a function of the corruption rate $\alpha$ and the proportion of the dataset in the fixed group $A$, denoted by $r_A$. 

Intuitively, this means that the smaller a group is, the easier it is for the adversary to make a fair hypothesis seem unfair for members of that group.


\begin{restatable}{theorem}{mainparity}\label{thm:mainparity}
For any hypothesis class $\mathcal{H}$ and distribution $ \dist = (\DA, \DB)$, a robust fair-ERM learner for the parity constraint in the Malicious Adversarial Model returns a hypothesis $\hhat \in \PQ$ such that 

$\error{\hhat} \leq O(\alpha)$ \\
where $h^*$ is the optimal classifier for the fair-ERM problem on the true distribution $\mathcal{D}$ with respect to hypothesis class $\mathcal{H}$.
\end{restatable}

This theorem states that a fair-ERM learner searching over the smoothed hypothesis class 
$\PQ$
returns a classifier that is within $\alpha$ of the accuracy of the best fair classifier in the original class $\mathcal{H}$. The full constructive proof can be found in the appendix \ref{proof:mainparity}.

The proof exhibits classifier $h \in \PQ$ that satisfies the desired guarantee. 
This classifier mostly behaves identically to $h^*$ but deviates with probability $p_A$ on samples from group $A$ (and with probability $p_B$ on samples from group $B$). We give an explicit assignment of these probability values $p_A$, $q_A$, $p_B$, $q_B$ in $[0,1]$ so that $h$ is perceived as fair by the learner. Then, we show that these values are small enough that the proportion of samples where $h(x) \neq h^*(x)$ is small ($O(\alpha)$). \emph{This is the best possible outcome in the malicious adversary model without fairness constraints \cite{kearns1988learning}.}

\subsection{Equal Opportunity}
Equal Opportunity \cite{hardt16} requires that the True Positive Rates of the classifier are equal across all the groups, that is, $P_{(x,y) \sim \DA} [h(x)=1 \mid y = 1] = P_{(x,y) \sim \DB} [h(x)=1 \mid y = 1]$. 
Similarly to Demographic Parity, we first provide bounds on how the fairness violation of a fixed hypothesis changes after the adversary corrupts an $\alpha$ proportion of the dataset. 
This is important because it gives an estimate of how much violation must be offset.

\begin{restatable}[TPR after corruption]{proposition}{corrupttpr}\label{prop:corrupttpr}
    Let $\widetilde{\mathcal{D}}$ be any corrupted distribution chosen by the adversary, and $h$ be a fixed hypothesis in $\mathcal{H}$. For a fixed group $A$, the following inequality bounds the change in True Positive Rate of $h$:
    \begin{equation}
        \left| \text{TPR}_A(h, \widetilde{\mathcal{D}}) - \text{TPR}_A(h, \dist) \right| \leq \frac{\alpha}{(1-\alpha) r_A^+ + \alpha }
    \end{equation}
    where $\text{TPR}_A(h, \dist) = P_{(x,y) \sim \DA} [h(x)=1 | y = 1]$ and $r_A^+ = P_{(x,y) \sim \mathcal{D}} [y = 1 \cap x \in A]$
\end{restatable}

This proposition provides an upper bound on the change to the true positive rate in group $A$ assigned by a fixed hypothesis $h$ in $\mathcal{H}$ after the dataset has been corrupted according to the Malicious Noise Model. 
The full proof can be found in the appendix \ref{proof:corrupttpr}. 

\emph{Since $\alpha \in [0,1]$, $O(\sqrt{\alpha})$ means larger (meaning worse) accuracy loss, compared $O(\alpha)$.}

The function that bounds the change in True Positive rate is similar to that of Demographic Parity with the proportional size of group A $r_A$ replaced with the proportion of the dataset that is positively labeled and in group A, $r_A^+$.
We will see that this slight change in dependence makes the robust learning problem more difficult and leads to a worse dependence on $\alpha$.

\begin{restatable}[Upper Bound]{theorem}{maineopp}\label{thm:maineopp}
For any hypothesis class $\mathcal{H}$ and distribution $ \dist = (\DA, \DB)$, a robust fair-ERM learner for the equal opportunity constraint in the Malicious Adversarial Model returns a hypothesis $\hhat$ such that 
$\error{\hhat} \leq O(\sqrt{\alpha})$
where $h^*$ is the optimal classifier for the fair-ERM problem on the true distribution $\mathcal{D}$ with respect to hypothesis class $\mathcal{H}$.
\end{restatable}

This theorem states that a fair-ERM learner, when applied with the smoothed hypothesis class $\PQ$, returns a classifier that is within $\sqrt{\alpha}$ of the accuracy of the best fair classifier in the original class $\mathcal{H}$. The full proof can be found in Appendix \ref{proof:mainparity}. 

In constructing a classifier $h \in \PQ$, we aim for it to behave mostly identically to $h^*$ but introduce deviations with probability $p_A$ for samples from group $A$ and probability $p_B$ for samples from group $B$. 
However, in the case of the Equal Opportunity fairness constraint, this approach, as used for Demographic Parity, does not work effectively. 
We observe that the amount of correction required for each group depends inversely on the true positive rate, which presents challenges when the true positive rate (TPR) is close to 0 or 1.

For example, suppose the classifier achieves a 95\% TPR for a fixed group. The adversary can manipulate the TPR to reach 100\% by corrupting only a few samples. 
Correcting this change and bringing the TPR back down to 95\% is an incredibly difficult task, similar to finding a \textit{needle in a haystack}, since the learner essentially has to identify the corrupted samples to do so.
In such cases, it might be easier for the learning algorithm to increase the TPR of the other groups from 95\% to 100\% instead. 

The tradeoff lies in equalizing the corrections that only transform the TPR of a fixed group to its original value versus the corrections that transform the TPR of other groups to match the TPR of the group with the most corruptions.

\begin{restatable}[Lower Bound]{theorem}{lowereopp}\label{thm:lowereopp}
There exists a distribution $ \dist = (\DA, \DB)$ and a malicious adversary of power $\alpha$ that guarantees that any hypothesis, $\hhat$, returned by an improper learner for the fair-ERM problem with the equal opportunity constraint satisfies the following:
$\error{\hhat} \geq \Omega(\sqrt{\alpha})$
where $h^*$ is the optimal classifier for the fair-ERM problem on the true distribution $\mathcal{D}$ with respect to a hypothesis class $\mathcal{H}$.
\end{restatable}

In this lower bound, under the given conditions, no proper or improper learner can achieve an error rate lower than a threshold that scales with the square root of the adversary's power. In other words, as the adversary becomes more powerful ($\alpha$ increases), the error rate of the hypothesis returned by an improper learner will unavoidably be at least on the order of $\sqrt{\alpha}$.

The proof of this lower bound result sets up a scenario reflecting the \textit{needle in the haystack} issue described earlier. We present a distribution with two groups, one of size $\sqrt{\alpha}$ and the other of size $1 - \sqrt{\alpha}$. We construct a hypothesis class where the optimal classifier has a high but not perfect true positive rate. 
Then we show that any improper learner must either suffer poor accuracy on the smaller group or lose $\Omega(\sqrt{\alpha})$ accuracy on the larger group. 
The full proof can be found in the Appendix \ref{proof:corrupttpr}.

\subsection{Equalized Odds}
Equalized Odds \cite{hardt16} is a fairness constraint that requires equalizing
True Positive Rates (TPRs) and False Positive Rates (FPRs) across different groups. 
This notion is very sensitive to the adversary's corrupted data and we exhibit a problematic lower bound, showing the adversary can force terrible performance.

The intuition is as follows; for a small group, the Adversary can set the Bayes Optimal TPR/FPRs rates of that group towards arbitrary values and so the learner must do the same on the larger group, regardless of their hypothesis class, forcing large error.
The full proof is in Appendix \ref{sec:eoddsproof}.
\begin{theorem}[Lower Bound]
\label{thm: Equalized Odds} 
For a learner seeking to maximize accuracy subject to satisfying Equalized Odds, an adversary with corruption fraction $\alpha$ can force an additional $\Omega(1)$ accuracy loss when compared to the performance of the optimal fair classifier on the true distribution.
\end{theorem}

\section{Main Results: Calibration}
\label{sec:calib}
In this section, we explore various notions of calibration \cite{dawid} for our model.
Calibration is a desirable property typically considered for classifiers, where predicted label probabilities should correspond to observed frequencies in the long run. For example, in weather forecasting, a well-calibrated predictor should have approximately 60\% of days with rain when it forecasts a 60\% chance of rain. This calibration requirement should hold for every predicted probability value output by the model.

Calibration has important fairness implications \cite{flores2016false, chouldechova2017fair, faircalib,multicalib} because a mis-calibrated predictor can lead to harmful actions in high-stakes settings, such as over-incarceration \cite{compassgender}. 
We show that varying the exact calibration requirements can substantially impact the model's accuracy loss when malicious noise is present in the training data.

In this section, we align closely  with \cite{faircalib}, where the learner seeks to maximize accuracy while ensuring the classifier is perfectly calibrated.
Throughout this paper, we have focused on binary classifiers, so in Section \ref{subsec: predparity} we consider a related notion called Predictive Parity \cite{chouldechova2017fair, flores2016false}, before considering calibration notions for hypotheses with output in $[0,1]$.

\subsection{ Predictive Parity Lower Bound}
\label{subsec: predparity}
\begin{definition}[Predictive Parity \cite{chouldechova2017fair}]
A binary classifier $h: \mathcal{X} \rightarrow \{0,1\}$ satisfies predictive parity if for groups A and B, $P_{x \sim \DA}[h(x)=1]>0$, $P_{x \sim \DB}[h(x)=1]>0$ and
\[P_{(x,y) \sim \DA} [y=1 | h(x)=1] = P_{(x,y) \sim \DB} [y=1 | h(x)=1] \]
\end{definition}
In later sections we consider other calibration notions.
Here we consider an adversary who is attacking a learner constrained by equal predictive parity when group sizes are \emph{imbalanced}.

\begin{theorem}
\label{thm:predparity}
    For a malicious adversary with corruption fraction $\alpha$, for Fair-ERM constrained to satisfy Predictive Parity, then there is no $h \in \PQ$ with less than $\Omega(1)$ error. 
\end{theorem}

The intuition for this statement is that imbalanced group size will allow the adversary to change the conditional mean substantially.
Below, we have an informal proof:
\begin{proof}[Proof Sketch:]
Suppose $P(x \in A)=1-\alpha$ and $P(x \in B)= \alpha$.
Observe that whatever the initial value of $P_{(x,y) \sim \DB} [y=1 | h(x)=1]$, the adversary can drive this value $P_{(x,y) \sim \mathcal\DBC} [y=1 | h(x)=1]$ to $50\%$ or below
by adding a duplicate copy of every natural example in group $B$ with the opposite label.

Since all of these points are information-theoretically indistinguishable, any hypothesis for group $B$ that makes any positive predictions incurs at least $50\%$ error and $1/2=P_{(x,y) \sim \mathcal\DBC} [y=1 | h(x)=1]$ calibration error.
Any classifier for group $A$ satisfying Predictive Parity will have to do the same, yielding our $\Omega(1)$ error.
\end{proof}

\subsection{Extension to Finer Grained Hypothesis Classes}
\label{subsec:calib}
A criticism of this lower bound might be that these calibration notions are very coarse and calibration is intended for fine-grained predictors, meaning those that have a finer grained discretization of the probabilities in $[0,1]$.
We now provide extensions for these lower bounds to real valued $\mathcal{H}$. 
Interestingly, we show if the learner can modify their `binning strategy', the learner can `decouple' the classifiers for the groups in the population and 
thus only suffer $O(\alpha)$ accuracy loss.
We adopt the version of calibration from \cite{faircalib}.
\begin{definition}[Calibration] \label{def:calib}
A classifier $h: \mathcal{X} \rightarrow [0,1]$ is Calibrated with respect to distribution $\mathcal{D}$ if 
\[\forall r \in [0,1], r= \mathbb{E}_{(x,y) \sim \mathcal{D} }[y=1| h(x)=r]\]
We will primarily focus on the discretized version of this definition where the classifier assigns every data point to one of $R$ bins, each with a corresponding label $r$, that partition $[0,1]$ dis-jointly. 
We will refer to this partition as $[R]$ with $r \in [R]$ corresponding to the prediction of a bin. 
\[ \forall r \in [R], r= \mathbb{E}_{(x,y) \sim \mathcal{D} }[y=1| h(x)=r] \]
\end{definition}
Calibration as a fairness requirements with demographic groups requires that the classifier $h$ is calibrated with 
respect to the group distributions $\DA$ and $\DB$ simultaneously. 
In the sections that follow when we say `calibrated' this always refers to calibration with respect to $\DA$ and $\DB$. 

\begin{theorem}
\label{thm:calib}
    The learner wants to maximize accuracy subject to using a calibrated classifier, $h: \mathcal{X} \rightarrow [R]$ where $[R]$ is a partition of $[0,1]$ into bins.
    
    The learner may modify the binning strategy after the adversary commits to a corruption strategy.
    Then an adversary with corruption fraction $\alpha$ can force at most $O(\alpha)$ excess accuracy loss over the non-corrupted optimal
    classifier. 
\end{theorem}

\newpage

\subsection{Parity Calibration}
Motivated by Theorem \ref{thm:calib}, we introduce a \emph{novel} fairness notion we call \emph{Parity Calibration}\footnote{We would note that this is initial discussion of a novel fairness constraint that arose naturally from considering Theorem \ref{thm:calib}. The idea is in some cases it might be more desirable to have a more sensitive calibration notion, hence we define Parity Calibration. This notion requires further study and analysis before deployment in sensitive contexts.}
Informally, this notion is a generalization of Statistical/Demographic parity \cite{dwork2012fairness} for the case of classifier with 
$R$ bins partitioning $[0,1]$.
\begin{definition}[Parity Calibration]
\label{def:paritycalib}
Classifier $h: \mathcal{X} \rightarrow [R]$, where $[R]$ is a partition of $[0,1]$ into labelled bins, satisfies
\emph{Parity Calibration} if the classifier is Calibrated (Definition \ref{def:calib}) \emph{and}
\begin{align*}
\forall r \in [R], P_{(x,y) \sim \DA} [h(x)=r] =  P_{(x,y) \sim \DB} [h(x)=r]
\end{align*} 
\end{definition}

\begin{theorem}
\label{thm:paritycalib} 
Consider a learner maximizing accuracy subject to satisfying Parity Calibration.
    The learner may modify the binning strategy after the adversary commits to a corruption strategy.
    Then an adversary with corruption fraction $\alpha$ can force $\Omega(1)$ excess accuracy loss over the non-corrupted optimal
    classifier. 
\end{theorem}


If the size of Group $B$ is $O(\alpha)$, then following a similar duplication strategy for Predictive Parity Theorem \ref{thm:predparity},
then the adversary can force Group $B$ to have an expected label of $50\%$, i.e.
$\forall x \in B, \mathbb{E}_{x \sim \DB}[y|x]=50\%$.
Thus, any classifier that is calibrated must assign all of Group $B$ to a $50\%$ bucket.
In order to satisfy \emph{Parity Calibration}, the classifier must do the same to Group $A$, yielding $50\%$ error on Group $A$.

\section{Discussion}
\label{sec: discussion}
We study Fair-ERM in the Malicious Noise model, and in some cases allow 
the learner to maintain optimal overall accuracy despite the signal in Group $B$ being almost entirely washed out.
In particular, we show that different fairness constraints have fundamentally different behavior in the presence of Malicious Noise, in terms of the amount of accuracy loss that a given level of Malicious Noise could cause a fairness-constrained learner to incur. 
The key to achieving our results, which are more optimistic than those in \cite{lampert}, is allowing for improper learners using the (P,Q)-randomized expansions of the given class $\mathcal{H}$.
The type of smoothness we create by using $\PQ$ seems to be a natural property that is likely shared by many natural hypothesis classes.

Fairness notions are motivated as a response to learned disparities when there is systemic error affecting one group. 
Fairness notions are supposed to mitigate this by ruling out classifiers that have worse performance on a sub-group. 
This can peg both classifiers at a lower level of performance in order to \emph{motivate} \cite{hardt16} improving the data collection or labelling process to obtain more reliable performance. 
However, it is also desirable that fairness constraints perform gracefully when subject to Malicious Noise, because fairness constraints will be used in contexts where the data is unreliable and noisy. 
This tension, exposed by our work, motivates 
ongoing work studying the sensitivity level of fairness constraints.

This work was supported in part by the National Science Foundation under grant CCF-2212968, by the
Simons Foundation under the Simons Collaboration on the Theory of Algorithmic Fairness, by the Defense
Advanced Research Projects Agency under cooperative agreement HR00112020003. The views expressed in
this work do not necessarily reflect the position or the policy of the Government and no official endorsement
should be inferred. Approved for public release; distribution is unlimited.

\bibliographystyle{apalike}
\bibliography{ref}

\onecolumn
\appendix



\section{Fairness Notions}
\label{fairtable}
\begin{center}
\begin{tabular}{ |p{4cm}||p{9cm}| }
 \hline
 \multicolumn{2}{|c|}{ Fairness Constraints}\\
 \hline
 Demographic Parity \cite{dwork2012fairness} & $ P_{(x,y) \sim \DA} [h(x)=1]= P_{(x,y) \sim \DB} [h(x)=1] $  \\
 \hline 
 Equal Opportunity \cite{hardt16} & $ P_{(x,y) \sim \DA} [h(x)=1| y=1 ]= P_{(x,y) \sim \DB} [h(x)=1| y=1 ] $ \\
 \hline
 Equalized Odds \cite{hardt16} &   $P_{(x,y) \sim \DA} [h(x)=1| y=1 ]= P_{(x,y) \sim \DB} [h(x)=1| y=1 ]$ and \\
  & $P_{(x,y) \sim \DA} [h(x)=1| y=0 ]= P_{(x,y) \sim \DB} [h(x)=1| y=0 ]$ \\
  \hline
  Predictive Parity \cite{chouldechova2017fair} & $P_{(x,y) \sim \DA} [y=1| h(x)=1 ]= P_{(x,y) \sim \DB} [y=1| h(x)=1 ]$ \\
\hline
Calibration\footnote{$h:\mathcal{X} \rightarrow [0,1]$} \cite{klein16,dawid} & $ \forall r \in [0,1],\quad  r = \mathbb{E}_{x,y \sim \mathcal{D}} [y|h(x)=r] $\\
\hline
\end{tabular}
\end{center}

\section{Proofs}

\corruptparity*

\begin{proof}[Proof of Proposition~\ref{prop:corruptparity}]\label{proof:corruptparity}
We want to bound the change in the proportion of positive labels assigned by $h$ when we move from the original distribution $\mathcal{D}$ to the corrupted distribution $\widetilde{\mathcal{D}}$. For a fixed group $A$, we can express the proportion of positive labels assigned by $h$ in $\widetilde{\mathcal{D}}$ in terms of the proportion of positive labels assigned by $h$ in $\mathcal{D}$ as follows:

\begin{equation}
    P_{(x,y) \sim \DAC} [h(x)=1] = \frac{(1-\alpha) P_{(x,y) \sim \DA} [h(x)=1] \cdot \RA + E_A}{(1-\alpha) \RA + \alpha_A}
\end{equation}

where $\alpha_A$ is the proportion of the data set that is corrupted and in group $A$ and $E_A$ is the proportion of the data set that is corrupted, in group $A$ and positively labeled by $h$.

Our goal is to obtain an upper bound on the difference between $P_{(x,y) \sim \DAC} [h(x)=1]$ and $P_{(x,y) \sim \DA} [h(x)=1]$.
We use the fact that $E_A \leq \alpha$ and $\alpha_A \leq \alpha$ to obtain the following upper bound:

\begin{equation}
    \left| P_{(x,y) \sim \DAC} [h(x)=1] - P_{(x,y) \sim \DA} [h(x)=1] \right| = \left| \frac{E_A - \alpha_A P_{(x,y) \sim \DA} [h(x)=1] }{(1-\alpha) \RA + \alpha_A} \right| \leq \frac{\alpha}{(1-\alpha) r_A + \alpha  }
\end{equation}
\end{proof}

\mainparity*

\begin{proof}[Proof of Theorem~\ref{thm:mainparity}]\label{proof:mainparity}
For $z \in \{A, B \}$, let $\normalF_z(h)$ and $\corruptF_z(h)$ denote the proportions of positive labels assigned by $h$ in group $z$ in the original and corrupted distributions respectively. That is, for group $A$, $\normalF_A(h) = P_{(x,y) \sim \DA} [ h (x)=1]$ and $\corruptF_A(h) = P_{(x,y) \sim \DAC} [ h (x)=1]$.
    It suffices to show that there exists $h \in \closure$ that satisfies the guarantees above. 
    Consider $\hstar \in \hclass$. By the realizability assumption 
    , $\hstar$ satisfies the parity constraint i.e $\normalF_A(h^*) = \normalF_B(h^*)$. 
    After the corruption, the parity violation of $h^*$, $|\corruptF_A(h^*) - \corruptF_B(h^*)|$ may increase. Now we define the following parameters ($p_z$ and $q_z$) for $z \in \{A, B \}$.
    \begin{equation}
        p_z = \begin{cases}
            \frac{\normalF_z(h^*) - \corruptF_z(h^*)}{1 - \corruptF_z(h^*)} & \text{if} \ \normalF_z(h^*) \geq \corruptF_z(h^*)\\
            \frac{\corruptF_z(h^*) - \normalF_z(h^*)}{\corruptF_z(h^*)} & \text{otherwise}\\
        \end{cases} \quad
        q_z = \begin{cases}
            1 & \text{if} \ \normalF_z(h^*) \geq \corruptF_z(h^*)\\
            0 & \text{otherwise}\\
        \end{cases}
    \end{equation}
    Now consider a hypothesis $\hhat$ that behaves as follows: Given a sample $x$:
    \begin{itemize}
        \item  If $x \in A$, with probability $p_A$, return label $q_A$. Otherwise return $h^* (x)$
        \item Similarly, if $x \in B$, with probability $p_B$, return label $q_B$. Otherwise return $h^* (x)$
    \end{itemize}
    $\hhat \in \PQ$ since it follows the definition of our closure model. We will now show that $\hhat$ satisfies the parity constraint in the corrupted distribution (i.e $\corruptF_A(\hhat) = \corruptF_B(\hhat)$). First, observe that for $z \in \{A, B \} $, if $\normalF_z(h^*) \geq \corruptF_z(h^*)$, then $\corruptF_z(\hhat) = \normalF_z(h^*)$. This is because
    \begin{align*}
        \corruptF_z(\hhat) 
        &= (1 - p_z) \corruptF_z(h^*) + p_z q_z \\
        &= \corruptF_z(h^*) + p_z(1 - \corruptF_z(h^*)) \\
        &= \corruptF_z(h^*) + \normalF_z(h^*) - \corruptF_z(h^*) \\
        &= \normalF_z(h^*)
    \end{align*}
    Similarly, if $\normalF_z(h^*) < \corruptF_z(h^*)$, then $\corruptF_z(\hhat) = \normalF_z(h^*)$. This is because
    \begin{align*}
        \corruptF_z(\hhat) 
        &= (1 - p_z) \corruptF_z(h^*) + p_z q_z \\
        &= \corruptF_z(h^*) + p_z(0 - \corruptF_z(h^*)) \\
        &= \corruptF_z(h^*) + \normalF_z(h^*) - \corruptF_z(h^*) \\
        &= \normalF_z(h^*)
    \end{align*}
    Thus, $\corruptF_A(\hhat) = \normalF_A(h^*) = \normalF_B(h^*) = \corruptF_B(\hhat)$. Therefore $\hhat$ satisfies the parity constraint in the corrupted distribution.
    
    We will now show that $\error{\hhat} \leq O(\alpha) $. Since $\hhat$ deviates from $\hstar$ with probability $p_A$ on samples from $A$, and with probability $p_B$ on samples from $B$, we only need to show that the proportion of samples such that $\hhat (x) \neq h^* (x)$ is small. Fix a group $z \in \{A, B\}$. If $\normalF_z (h^*) \geq \corruptF_z (h^*)$, then with probability $p_z = \frac{\normalF_z (h^*) -\corruptF_z (h^*)}{1 - \corruptF_z (h^*)}$, $\hhat$ returns a positive label for samples in group $z$. Thus, the expected proportion of samples in group $z$ such that $\hhat (x) \neq h^* (x)$ is $p_z$ times the proportion of negative labelled samples (by $h^*$) in group $z$ (since those get flipped to positive).
    \begin{align*}
        \EE_{x \in z} [\one (\hhat (x) \neq h^* (x))] &= p_z \cdot P_{(x,y) \sim \dist } [x \in z] (1 - \corruptF_z (h^*) ) \\
        &= \frac{\normalF_z (h^*) -\corruptF_z (h^*)}{1 - \corruptF_z (h^*)} \cdot P_{(x,y) \sim \dist } [x \in z] (1 - \corruptF_z (h^*) ) \\
        &= (\normalF_z (h^*) -\corruptF_z (h^*)) \cdot P_{(x,y) \sim \dist } [x \in z] 
    \end{align*}
    Similarly, if $\corruptF_z (h^*) > \normalF_z (h^*)$, then with probability $p_z = \frac{\corruptF_z (h^*) -\normalF_z (h^*)}{\corruptF_z (h^*)}$, $\hhat$ returns a negative label. Thus, the expected proportion of samples in group $z$ such that $\hhat (x) \neq h^* (x)$ is $p_z$ times the proportion of positively labelled samples (by $h^*$) in group $z$ (since those get flipped to negative).
    \begin{align*}
        \EE_{x \in z} [\one(\hhat (x) \neq h^* (x))] &= p_z \cdot P_{(x,y) \sim \dist } [x \in z] \cdot \corruptF_z (h^*) \\
        &= \frac{\corruptF_z (h^*) -\normalF_z (h^*)}{\corruptF_z (h^*)} \cdot P_{(x,y) \sim \dist } [x \in z] \cdot \corruptF_z (h^*)  \\
        &= (\corruptF_z (h^*) -\normalF_z (h^*)) \cdot P_{(x,y) \sim \dist } [x \in z] 
    \end{align*}
    Therefore, the expected total number of samples such that $\hhat (x) \neq h^* (x)$ across the entire distribution is bounded as follows:
    \begin{align*}
        \mathbb{E}_{(x,y) \sim \mathcal{D}} \ [\one (\hhat (x) \neq h^* (x))] 
        &= \sum_{z \in \{ A, B \}} |\corruptF_z (h^*) -\normalF_z (h^*)| \cdot P_{(x,y) \sim \dist } [x \in z] \\ 
        &\leq \sum_{z \in \{ A, B \}} \frac{\alpha}{(1- \alpha) P_{(x,y) \sim \dist } [x \in z] + \alpha} \cdot P_{(x,y) \sim \dist } [x \in z] \\ \intertext{by proposition \ref{prop:corruptparity}} 
        &\leq \frac{2\alpha}{(1- \alpha)}
    \end{align*}
    Note that even though the adversary can choose a different distribution at each timestep, we can wlog assume the adversary chooses the same distribution $\widetilde{D}$ where the quantity $|\corruptF_z (h^*) -\normalF_z (h^*)|$ is maximized at every timestep, as in Proposition \ref{prop:corruptparity}.
    Although the model in \cite{kearns1988learning} is slightly weaker than \cite{lampert}, this theorem holds in full generality for both models where we replace the difference $|\corruptF_z (h^*) -\normalF_z (h^*)|$ with the bounds from Lemma 2 of \cite{lampert}. The dependence on $\alpha$ remains the same in both cases.
\end{proof}

\subsection{Equal Opportunity}

\corrupttpr*

\begin{proof}[Proof of Proposition~\ref{prop:corrupttpr}]\label{proof:corrupttpr}
For a fixed group $A$, the TPR of $h$ in $\widetilde{\mathcal{D}}$ can be expressed in terms of the TPR of $h$ in the original distribution $\mathcal{D}$ as follows:
\begin{equation}
        \text{TPR}_A(h, \widetilde{\mathcal{D}}) = \frac{(1-\alpha) \text{TPR}_A(h, \mathcal{D}) \cdot \RA + E_A^+}{(1-\alpha) \RA + \alpha_A^+}
\end{equation}
where $\alpha_A$ is the proportion of the data set that is corrupted and in group $A$ and $E_A^+$ is the proportion of the data set that is corrupted, in group $A$, is positive, and is predicted as positive by $h$.
Thus,
\begin{equation}
    \left| \text{TPR}_A(h, \widetilde{\mathcal{D}}) - \text{TPR}_A(h, \mathcal{D}) \right| = \left| \frac{E_A - \alpha_A \text{TPR}_A(h, \mathcal{D}) }{(1-\alpha) \RA + \alpha_A} \right| \leq \frac{\alpha}{(1-\alpha) r_A^+ + \alpha }
\end{equation}
since $E_A \leq \alpha$ and $\alpha_A \leq \alpha$
\end{proof}


\maineopp*

\begin{proof} [Proof of Theorem~\ref{thm:maineopp}]\label{proof:maineopp}

We will use Proposition \ref{prop:corrupttpr} and the assumption we introduced in Section \ref{sec:prelim}, Equation \ref{ass:raplus} to show this statement.
 To prove the statement, it suffices to show that there exists $h \in \PQ$ that satisfies the guarantees above. 
Consider $\hstar \in \hclass$. By the realizability assumption, $\hstar$ satisfies the equal opportunity constraint i.e $\text{TPR}_A(h^*, \mathcal{D}) = \text{TPR}_B(h^*, \mathcal{D})$. 
After the corruption, the equal opportunity violation of $h^*$, $|\text{TPR}_A(h^*, \widetilde{\mathcal{D}}) - \text{TPR}_B(h^*, \widetilde{\mathcal{D}})|$ may increase. Now we define the following parameters ($p_z^i$ and $q_z^i$) for $i, z \in \{A, B \}$. 
\begin{equation}\label{eq:tpr-prob}
    p_z^i = \begin{cases}
        \frac{\corruptF_i(h^*) - \corruptF_z(h^*)}{1 - \corruptF_z(h^*)} & \text{if} \ \corruptF_i(h^*) \geq \corruptF_z(h^*)\\
        \frac{\corruptF_z(h^*) - \corruptF_i(h^*)}{\corruptF_z(h^*)} & \text{otherwise}\\
    \end{cases} \quad
    q_z^i = \begin{cases}
        1 & \text{if} \ \corruptF_i(h^*) \geq \corruptF_z(h^*)\\
        0 & \text{otherwise}\\
    \end{cases}
\end{equation}
One can think of the parameter $p_z^i$ as the proportion of samples in group $z$ whose outcomes needs to be changed in order to match the true positivity rate of group $i$.
    Now consider two hypotheses $\hhat_i$ for $i \in \{ A, B\}$ that behave as follows: Given a sample $x$:
    \begin{itemize}
        \item  If $x \in A$, with probability $p_A^i$, return label $q_A^i$. Otherwise return $h^* (x)$
        \item Similarly, if $x \in B$, with probability $p_B^i$, return label $q_B^i$. Otherwise return $h^* (x)$
    \end{itemize}
One can think of $\hhat_i$ as a hypothesis that deviates from $h^*$ on every other group to make their true positive rate on the corrupted distribution match that of group $i$.
Observe that $\hhat_i \in \PQ$ for $i \in \{A, B\}$ since it follows the definition of our closure model $\PQ$. We will now show that $\hhat_i$ for $i \in \{ A, B\}$ satisfies the True Positive Rate constraint on the corrupted distribution (i.e $\corruptF_A(\hhat_i) = \corruptF_B(\hhat_i)$ for fixed $i \in \{ A, B\}$). First, observe that for $z \in \{A, B \} $, if $\corruptF_i(h^*) \geq \corruptF_z(h^*)$, then $\corruptF_z(\hhat_i) = \corruptF_i(h^*)$. This is because
    \begin{align*}
        \corruptF_z(\hhat_i) 
        &= (1 - p_z) \corruptF_z(h^*) + p_z q_z \\
        &= \corruptF_z(h^*) + p_z(1 - \corruptF_z(h^*)) \\
        &= \corruptF_z(h^*) + \corruptF_i(h^*) - \corruptF_z(h^*) \\
        &= \corruptF_i(h^*)
    \end{align*}
    Similarly, if $\corruptF_i(h^*) < \corruptF_z(h^*)$, then $\corruptF_z(\hhat) = \corruptF_i(h^*)$. This is because
    \begin{align*}
        \corruptF_z(\hhat) 
        &= (1 - p_z) \corruptF_z(h^*) + p_z q_z \\
        &= \corruptF_z(h^*) + p_z(0 - \corruptF_z(h^*)) \\
        &= \corruptF_z(h^*) + \corruptF_i(h^*) - \corruptF_z(h^*) \\
        &= \corruptF_i(h^*)
    \end{align*}
    Thus, $\corruptF_A(\hhat_i) = \corruptF_i(h^*) = \corruptF_B(\hhat_i)$. Therefore $\hhat_i$ for $i \in \{ A, B\}$ satisfies the Equal Opportunity Constraint on the corrupted distribution.
    We will now show that the existence of at least one $\hhat_i$ for $i \in \{A, B\}$ satisfies \[ \error{\hhat} \leq O(\sqrt{\alpha}) \].
    Since $\hhat_i$ deviates from $\hstar$ with probability $p_A^i$ on samples from $A$, and with probability $p_B^i$ on samples from $B$, it suffices to show that $p_A^i \cdot r_A + p_B^i \cdot r_B$ is $O(\sqrt{\alpha})$ for $i \in \{A, B\}$. 
    
    We consider the following cases:
\begin{enumerate}
    \item Suppose wlog $r_B \leq \frac{\sqrt{\alpha}}{1 - \sqrt{\alpha}}$. Then $\hat{h}_{B}$ satisfies the guarantee. This is because $p_A^B = 0$ (by equation~\ref{eq:tpr-prob} ) and $p_B^B \leq 1$. Thus, $p_A^B \cdot r_A + p_B^B \cdot r_B$ is $O(\sqrt{\alpha})$. 

    \item If instead $\min (r_A, r_B) > \frac{\sqrt{\alpha}}{1 - \sqrt{\alpha}}$. wlog let $B$ be a group with the highest true positive rate greater than 0.5 or the smallest true positive rate less than 0.5. At least one group must satisfy this constraint. If $B$ has the highest true positive rate greater than 0.5, then 
    \begin{align*}
    p_B^A &= \frac{\corruptF_B (h^*) - \corruptF_A (h^*)}{\corruptF_B (h^*)} \\
    &\leq \frac{\corruptF_B (h^*) - F_B (h^*) + F_A (h^*) - \corruptF_A (h^*)}{0.5} \intertext{since $\corruptF_B (h^*) \geq 0.5$ and by realizability assumption $\normalF_B (h^*) = \normalF_A (h^*)$}
    &\leq 2 |\corruptF_B (h^*) - F_B (h^*)| + 2 |\normalF_A (h^*) - \corruptF_A (h^*)| \\ \intertext{by proposition~\ref{prop:corrupttpr} and the assumption in Equation \ref{ass:raplus}.}
    &\leq O (\sqrt{\alpha})
    \end{align*} 
    Thus, $p_A^A \cdot r_A + p_B^A \cdot r_B$ is at most $O(\sqrt{\alpha})$
    The case where $B$ has the smallest true positive rate follows similarly.
\end{enumerate}
Similar to the proof of Theorem~\ref{thm:mainparity}, we can assume wlog the adversary chooses the same distribution $\widetilde{D}$ where the quantity $|\corruptF_z (h^*) -\normalF_z (h^*)|$ is maximized at every timestep, as in Proposition \ref{prop:corruptparity}.
Although the model in \cite{kearns1988learning} is slightly weaker than \cite{lampert}, this theorem holds in full generality for both models where we replace the difference $|\corruptF_z (h^*) -\normalF_z (h^*)|$ with the bounds from Lemma 5 of \cite{lampert}. The dependence on $\alpha$ remains the same in both cases.
\end{proof}



\lowereopp*

\begin{proof} [Proof of Theorem~\ref{thm:lowereopp}]\label{proof:lowereopp}
We will show a distribution and a malicious adversary of power $\alpha$ such that any hypothesis returned by a learner incurs at least $\sqrt{\alpha}$ expected excess error.
The distribution $\mathcal{D}$ will be such that $P_{x \sim \mathcal{D}}[x \in B]= \Omega(\sqrt{\alpha})$. This distribution will be supported on exactly four points $x_1 \in A, x_2 \in A, x_3 \in B, x_4 \in B$ with labels $y_1 = +, y_2 = -, y_3 = +, y_4 = -$. We also have that 
$$P_{x,y \sim \mathcal{D}}[x = x_1, y = +] = P_{x, y \sim \mathcal{D}}[x = x_2, y = -] = \frac{1 - \sqrt{\alpha}}{2}$$ 
and 
$$P_{x, y \sim \mathcal{D}}[x = x_3, y = +] = P_{x,y \sim \mathcal{D}}[x = x_4, y = -] = \frac{\sqrt{\alpha}}{2}$$
That is, each group has equal proportion of positives and negatives.

The adversary commits to a poisoning strategy that places positive examples from Group $B$ into the negative region of the optimal classifier. That is, the adversary changes the original distribution $\mathcal{D}$ so that 
$$P_{x, y \sim \mathcal{D}}[x = x_1, y = +] = P_{x, y \sim \mathcal{D}}[x = x_2, y = -] = \frac{(1-\alpha)(1 - \sqrt{\alpha})}{2}$$
$$P_{x, y \sim \mathcal{D}}[x = x_3, y = +] =
P_{x, y \sim \mathcal{D}}[x = x_4, y = -] = \frac{(1-\alpha)\sqrt{\alpha}}{2}$$ and $P_{x, y \sim \mathcal{D}}[x = x_4, y = +] = \alpha$

We assume the perfect classifier is in the hypothesis class.
Now fix a classifier $h$ returned by a learner. This classifier must satisfy equal opportunity. Let $p_1, p_2, p_3, p_4$ be the probability that 
$h$ classifies $x_1, x_2, x_3, x_4$  as positive, respectively. 
Observe that $\widetilde{\text{TPR}}(h_A) = p_1$ and $\widetilde{\text{TPR}}(h_B) = 1 - (1 - p_4) \alpha' - (1-p_3)(1 - \alpha')$ where $\alpha' = \frac{2\sqrt{\alpha}}{(1 - \alpha) + 2\sqrt{\alpha}}$. The latter is due to the samples $(x_4, +)$ which the adversary added to the distribution. 
The adversary added an $\alpha$ amount which turned out to be an $\alpha'$ proportion of the positives in $B$. 
Since this classifier satisfies equal opportunity on the corrupted distribution, it must be the case that $p_1 = 1 - (1 - p_4) \alpha' - (1-p_3)(1 - \alpha')$. Thus, $(1 - p_1) \geq (1 - p_4) \alpha'$.
The error of $h$ on the original distribution is therefore
\begin{align*}
& (1 - p_1 + p_2) \frac{(1 - \sqrt{\alpha})}{2} + (1 - p_3 + p_4) \frac{\sqrt{\alpha}}{2} \\
\geq & \ (1 - p_1) \frac{(1 - \sqrt{\alpha})}{2} + p_4 \frac{\sqrt{\alpha}}{2} \\ \intertext{by the equal opportunity constraint}
\geq & \ (1 - p_4) \alpha' \frac{(1 - \sqrt{\alpha})}{2} + p_4 \frac{\sqrt{\alpha}}{2} \\
= & \ (1 - p_4) \cdot \frac{2\sqrt{\alpha}}{(1 - \alpha) + 2\sqrt{\alpha}} \cdot \frac{(1 - \sqrt{\alpha})}{2} + p_4 \frac{\sqrt{\alpha}}{2} \\
\geq & \ (1 - p_4) \frac{\sqrt{\alpha}}{2} + p_4 \frac{\sqrt{\alpha}}{2}  \geq \Omega (\sqrt{\alpha}) 
\end{align*}
\end{proof}

\section{Equalized Odds}
\label{sec:eoddsproof}
Now we will consider Equalized Odds.

\begin{proof}[Equalized Odds Proof  of $\Omega(1)$ accuracy loss:] \label{proof:eodds}
it suffices to exhibit a `bad' distribution and matching corruption strategy; which we exhibit below.

\begin{enumerate}
\item Say Group A has $1-\alpha$ of the probability mass i.e. $P_{(x,y) \sim \calD}[x \in A] \geq 1-\alpha$ and thus $P_{(x,y) \sim \calD}[x \in B] \leq \alpha$.
\item The positive fraction for each group under distribution $\mathcal{D}$ is $P_{(x,y) \sim \DA}[y=1]=P_{(x,y) \sim D_B}[y=1]=\frac{1}{2}$
\item Since $P_{(x,y)\sim \calD}[x \in B] \leq \alpha$, the adversary has sufficient corruption budget such that they can inject a duplicate copy of each example in B but with the opposite label.   
That is, for each example x in Group B in the training set, the adversary adds another identical example but with the opposite label.
\end{enumerate}

This adversarial data ensures that on Group $B$, any hypothesis $h$ (of any form) will now satisfy 
\[ P_{x \sim \hat{\mathcal{D}}_{B}}[h(x)=1 | y=1] = P_{x \sim   \hat{\mathcal{D}}_{B}}[h(x)=1  | y=0] = p  \] for some value $p \in [0,1]$
due to the indistinguishable duplicated examples; i.e. the hypothesis can choose how often to accept examples [e.g. increase or decrease $p$] but it cannot distinguish positive/negative examples in Group $B$.

Note that we can select $p$ using some arbitrary $h$ but that randomness does not help us.
Observe that similarly, the True Negative/False Negative Rates on Groyp $B$ must be $1-p$.

Since $A$ is evenly split among positive and negative and we must satisfy Equalized Odds, 
this means that our error rate on group A is
\begin{align*}
& P_{(x,y) \sim \DA}[ h(x) \neq y ] =  P_{(x,y) \sim \DA}[ h(x) \neq y \cap y=1] +  P_{(x,y) \sim \DA}[ h(x) \neq y \cap y=0] \\
& =P_{(x,y) \sim \DA}[ h(x) \neq 1 | y =1  ]P[y=1] + P_{(x,y) \sim \DA}[ h(x) \neq 0 | y = 0  ]P[y=0]  \\
& = P_{(x,y) \sim \DA}[ h(x) \neq 0 | y =1  ]P[y=1] + P_{(x,y) \sim \DA}[ h(x) = 1 | y = 0  ]P[y=0] \\
& = (1-TPR_{A}) \frac{1}{2} + FPR_{A} \frac{1}{2} \\
& =(1-p)(\frac{1}{2}) + p(\frac{1}{2})= \frac{1}{2}
\end{align*}
So, the adversary has forced us to have $50 \%$ error on group A which yeilds the result.
\end{proof}

\section{Calibration Proofs}

\begin{proof}[Proof of Theorem ~\ref{thm:predparity}, Predictive Parity Lower Bound]  \label{proof:predparity}


To show that Predictive Parity requires $\Omega(1)$ error when the adversary has corruption budget $\alpha$, even with our hypothesis class $\PQ$, 
it suffices to exhibit a `bad' distribution and matching corruption strategy; which we exhibit below. 

Recall that we require that $P_{x \sim \DA}[h(x)=1]>0$ and $P_{x \sim \DB}[h(x)=1]>0$. 
This is to avoid the case where the learner rejects all points from Group $B$.

\begin{enumerate}
\item Assume that group A has $1-\alpha$ of the probability mass i.e. $P_{(x,y) \sim \calD}[x \in A] \geq 1-\alpha$ and thus $P_{(x,y) \sim \calD}[x \in B] \leq \alpha$.
\item The positive fraction for each group under distribution $\mathcal{D}$ is $P_{(x,y) \sim \DA}[y=1]=P_{(x,y) \sim D_B}[y=1]=\frac{1}{2}$
\item Since $P_{(x,y) \sim \calD}[x \in B] \leq \alpha$, the adversary has sufficient corruption budget such that they can a duplicate copy of each example in B but with the opposite label.   
That is, for each example x in Group B in the training set, the adversary adds another identical example but with the opposite label.
\end{enumerate}

This adversarial data ensures that on Group $B$, any hypothesis $h$ (of any form) will now satisfy 
\[  P_{(x,y) \sim \hat{\mathcal{D}}_{B}}[y=1 | h(x)=1 ] = P_{(x,y) \sim   \hat{\mathcal{D}}_{B}}[y=0  | h(x)=0]=\frac{1}{2} \]
due to the indistinguishable duplicated examples.  
So, for Group A, to satisfy Predictive Parity, both these terms must also equal $\frac{1}{2}$ and induce $50 \%$ error on Group $A$.
\end{proof}

\begin{proof}[Proof of Theorem~\ref{thm:calib}, Calibration $O(\alpha)$.]\label{proof:calib}

In order to prove this statement, we consider $h^{*}$ which is the Bayes Predictor $h^{*} = \mathbb{E}[y|x]$, but using some finite binning scheme $[R]$.
Clearly $h^{*}$ is calibrated on natural data and $h^{*}: \mathcal{X} \rightarrow [R]$. 

We  will show how to modify $h^{*}$ to still satisfy the fairness constraint on the corrupted data without losing too much accuracy, regardless
of the adversarial strategy.

In the case of Calibration,  we will do this by just separately re-calibrating each group. 

Let $[\hat{R}]:=[R]$.
We will now modify $[\hat{R}]$ from $[R]$ to be calibrated on the malicious data.

That is;
For each group $z$ (i.e $z=A$ or $z=B$), for each bin $r \in [R]$ (i.e., ${x: h^{*}(x)=r}$), we create a new bin if there is no bin in $[R]$ with value  $\hat{r} = E_{(x,y) \sim \mathcal{D}_z}[y | h^{*}(x)=r]$.  

That is, we define $\hat{h}(x) = \hat{r}$ for all $x \in g$ such that $h^{*}(x)=r$.

Observe that by construction, $\hat{h}$ is calibrated separately for each group, so it is calibrated overall.  We just need to analyze the excess error of $\hat{h}$ compared to $h^{*}$. 
We will show this is only
$O(\alpha)$.

Observe that increase in expected error is how much that bin is shifted from the true probability $h^{*}(x)$.

For each bin $r \in[R]$, the shift in 
$|r - \hat{r}|$ is at most the fraction of points in the bin that are malicious noise. 
Let $x \in MAL$ mean point $x$ is a  corrupted point.

Then 
\begin{align*}
& \mathbb{E}_{x \sim \mathcal{D}}[h^{*}(x)-\hat{h}(x)] \leq \sum_{r \in [R]} P[x \in r] |r-\hat{r}| \\
& = \sum_{r \in [R]} P[x \in r] \frac{P[x \in r \cap x \in MAL]}{P[x \in r]} \\
& \leq \sum_{r \in [R]} P[x \in r \cap x \in MAL] = O(\alpha) \quad \text{ Definition of Malicious Noise Model}
\end{align*}
Note that this is considering $L1$ error, accuracy loss is less than for $L2$ error, immediate for since $\alpha \in [0,1)$.

\end{proof}

\section{Minimax Fairness}
\label{subsec:minimaxfair}
In this Section, we will briefly and informally consider Minimax Fairness.
Introduced in \cite{minimaxfair} this notion  optimizes for a different objective. 

Using their notation ($\epsilon_k = \mathbb{E}_{(x,y) \sim \mathcal{D}_k}[h(x) \neq y]$ or group-wise error) with a groupwise max error bound of $1> \gamma > 0$
\begin{align*} 
h^{*} = \argmin_{h \in \Delta{H}} \quad  \mathbb{E}_{(x,y) \sim \mathcal{D}}[h(x)\neq y] \\
\max_{1 \leq k \leq K}  \epsilon_{k} (h) \leq \gamma
\end{align*}
Letting $OPT$ refer to the value of solution of the optimization problem, the learning goal is to find an $h$ that is $\epsilon$-approximately optimal for the mini-max objective, meaning that $h$ satisfies: 
\[ max_{k} \epsilon_{k}(h) \leq OPT + \epsilon\]

Observe that if the goal of the learner is compete with the value of $OPT$ on the unmodified data, in our malicious noise model this objective is
ineffective since if one group is of size $O(\alpha)$, the adversary can always drive the error rate on that group $\Omega(1)$.

This model seems incompatible with malicious noise due to the sensitivity of minimax fairness to small groups. 

Observe that the Minimax Fairness framework includes Equalized Error
rates as a special case.





\end{document}